\documentclass[sigconf]{acmart}

\AtBeginDocument{%
}

\copyrightyear{2024}
\acmYear{2024}
\setcopyright{rightsretained}
\acmConference[AISec '24]{Proceedings of the 2024 Workshop on Artificial Intelligence and Security }{October 14--18, 2024}{Salt Lake City, UT, USA}
\acmBooktitle{Proceedings of the 2024 Workshop on Artificial Intelligence and Security (AISec '24), October 14--18, 2024, Salt Lake City, UT, USA}
\acmDOI{10.1145/3689932.3694769}
\acmISBN{979-8-4007-1228-9/24/10}

\makeatletter
\gdef\@copyrightpermission{
 \begin{minipage}{0.3\columnwidth}
  \href{https://creativecommons.org/licenses/by/4.0/}{\includegraphics[width=0.90\textwidth]{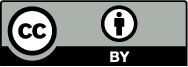}}
 \end{minipage}\hfill
 \begin{minipage}{0.7\columnwidth}
  \href{https://creativecommons.org/licenses/by/4.0/}{This work is licensed under a Creative Commons Attribution International 4.0 License.}
 \end{minipage}
 \vspace{5pt}
}
\makeatother

\usepackage{xcolor}
\usepackage{booktabs} 
\usepackage{xspace}
\usepackage{amsmath}
\usepackage{url}
\usepackage{multirow}
\usepackage{array}
\usepackage{eqparbox}
\usepackage{algorithm}
\usepackage{algorithmic}
\usepackage{subcaption}

\usepackage{enumitem}
\usepackage{rotating}
\usepackage{wrapfig}

\newcommand{\parheading}[1]{\medskip\noindent\textbf{\textit{#1}}\hspace{2pt}}

\renewcommand{\AE}[0]{AE\xspace}
\newcommand{\re}[0]{RE}
\newcommand{\ml}[0]{ML\xspace}
\newcommand{\dnn}[0]{DNN\xspace}
\newcommand{\imagenet}[0]{ImageNet\xspace}
\newcommand{\cifarten}[0]{CIFAR-10\xspace}
\newcommand{\fgsm}[0]{FGSM\xspace}
\newcommand{\ifgsm}[0]{I-\fgsm}
\newcommand{\mifgsm}[0]{MI-\fgsm}

\newcommand{\lpnorm}[1]{$\ell_{#1}$\xspace}
\newcommand{\chanshuffle}[0]{CS}
\newcommand{\greyscale}[0]{GS}
\newcommand{\colorjitter}[0]{CJ}
\newcommand{\fancyPCA}[0]{fPCA}
\newcommand{\cutout}[0]{CutOut}
\newcommand{\cutmix}[0]{CutMix}
\newcommand{\randerase}[0]{RE}
\newcommand{\sharpen}[0]{Sharp}
\newcommand{\neuraltrans}[0]{NeuTrans}
\newcommand{\autoaug}[0]{AutoAugment}
\newcommand{\dstonly}[0]{DST\xspace}

\newcommand{\dst}[0]{\dstonly{}}

\newcommand{\admix}[0]{Admix\xspace}
\newcommand{\admixdt}[0]{\admix{}-DT}
\newcommand{\SDTM}[0]{\dst}
\newcommand{\ASDTM}[0]{\admixdt}
\newcommand{\ultimatecombo}[1]{\textsc{UltComb\textsubscript{#1}}\xspace}
\newcommand{\bestcombo}[0]{\ultimatecombo{Base}}
\newcommand{\bestcomboadv}[0]{\ultimatecombo{Gen}}
\newcommand{\advancedultimatecombo}[0]{\ultimatecombo{Gen}}
\newcommand{\gsdt}[0]{\greyscale-DST}
\newcommand{\undp}[0]{UNDP-DT}
\newcommand{\subsetattack}[1]{\bestcomboadv-#1\xspace}

\newcommand{\gdst}[0]{\gsdt}
\newcommand{\bitred}[0]{Bit-Red\xspace}
\newcommand{\nrp}[0]{NRP\xspace}
\newcommand{\randsmooth}[0]{RS\xspace}
\newcommand{\arandsmooth}[0]{ARS\xspace}
\newcommand{\TRS}[0]{TRS\xspace}
\newcommand{\sdtv}[0]{VMI-DST}

\newcommand{\figref}[1]{Fig.~\ref{#1}}

\newcommand{\tabref}[1]{Tab.~\ref{#1}}
\newcommand{\twotabrefs}[2]{Tabs.~\ref{#1}--\ref{#2}}
\newcommand{\secref}[1]{\S\ref{#1}}
\newcommand{\secrefs}[2]{\S\ref{#1}--\ref{#2}}
\newcommand{\appref}[1]{App.~\ref{#1}}
\newcommand{\algoref}[1]{Alg.~\ref{#1}}

\newcommand{\vit}[0]{ViT}
\newcommand{\mnas}[0]{MNAS }

\begin{document}

\title{The Ultimate Combo: Boosting Adversarial Example Transferability
  by Composing Data Augmentations}

\author{Zebin Yun}
\email{zebinyun@mail.tau.ac.il}
\affiliation{%
  \institution{Tel Aviv University}
  \city{Tel Aviv}
  \country{Israel}
}
\author{Achi-Or Weingarten}
\email{achi.wgn@gmail.com}
\affiliation{%
  \institution{Weizmann Institute of Science}
  \city{Rehovot}
  \country{Israel}
}
\author{Eyal Ronen}
\authornote{Corresponding authors.}
\email{eyalronen@tauex.tau.ac.il}
\affiliation{%
  \institution{Tel Aviv University}
  \city{Tel Aviv}
  \country{Israel}
}
\author{Mahmood Sharif}
\authornotemark[1]
\email{mahmoods@tauex.tau.ac.il}
\affiliation{%
  \institution{Tel Aviv University}
  \city{Tel Aviv}
  \country{Israel}
}

\renewcommand{\shortauthors}{Zebin Yun, Achi-Or Weingarten, Eyal Ronen, \& Mahmood Sharif}

\begin{abstract}

To help adversarial examples %
generalize from surrogate
machine-learning (\ml) models to targets, certain
transferability-based black-box evasion attacks incorporate data
augmentations (e.g., random resizing). Yet, prior work has explored
limited augmentations and their composition. To fill the gap, we
systematically studied how data augmentation affects 
transferability. Specifically, we explored 46 augmentation techniques
originally proposed to help \ml{} models generalize
to unseen benign samples, and assessed how they impact
transferability, when applied individually or composed. Performing
exhaustive search on a small subset of augmentation techniques and
genetic search on all techniques, we identified augmentation
combinations that help promote transferability.
Extensive experiments with the \imagenet{} and \cifarten{} datasets
and 18 models showed that simple color-space augmentations (e.g.,
color to greyscale) attain high transferability when combined with
standard augmentations. Furthermore,
we discovered that composing augmentations impacts transferability
mostly monotonically (i.e., more augmentations $\rightarrow$
$\ge$transferability). We also found
that the best composition
significantly outperformed the state of the art (e.g., 91.8\%
vs.\ $\le$82.5\% average transferability to adversarially trained
targets on \imagenet{}). Lastly, our 
theoretical analysis, backed by empirical evidence, intuitively
explains why certain augmentations promote transferability.

\end{abstract}

\ccsdesc[500]{Computing methodologies~Neural networks}
\ccsdesc[300]{Security and privacy~Software and application security}

\keywords{Adversarial Examples, Transferability, Neural Networks}

\maketitle

\section{Introduction}

Adversarial examples (\AE{}s)---variants of benign inputs minimally perturbed
to induce misclassification at test time---have emerged as a profound
challenge to machine learning (\ml{})
~\cite{biggio2013evasion,   szegedy2013intriguing}, calling its use in security- and
safety-critical systems into question (e.g.,~\cite{eykholt2018robust}). 
Many attacks have been proposed to generate \AE{}s in white-box
settings, where adversaries are 
familiar with all the particularities of the attacked
model~\cite{papernot16limitations}. By contrast, black-box attacks
enable evaluating the vulnerability of \ml{} in realistic settings, 
without access to the model~\cite{papernot16limitations}.

Attacks exploiting the transferability-property of
\AE{}s~\cite{szegedy2013intriguing} have received special
attention. Namely, as \AE{}s produced against one model are often
misclassified by others, 
transferability-based attacks produce \AE{}s against surrogate
(a.k.a. substitute) white-box models to mislead black-box ones. To
measure the risk of \AE{}s in black-box settings accurately,
researchers have proposed varied methods to enhance
transferability (e.g.,~\cite{lin2019nesterov, liu2016delving, li2023towards}). %
Notably, attacks using data augmentation, such as
translations~\cite{dong2019evading} and scaling of pixel
values~\cite{lin2019nesterov}, as a means to improve the
generalizability of \AE{}s across models have accomplished
state-of-the-art transferability rates. Still,  previous
transferability-based attacks have studied only six augmentation
methods (see~\secref{sec:dataaug:prev}), out of many proposed in the
data-augmentation literature, %
primarily for 
reducing model overfitting~\cite{shorten2019survey}. Hence, the extent to which different
data-augmentation types boost transferability, either individually or
when combined, remains largely unknown.

To fill the gap, we conducted a systematic study of how
augmentation methods influence transferability. Specifically,
alongside techniques considered in previous work, we studied how
46 augmentation techniques pertaining to seven categories impact
transferability when applied individually or composed
(\secref{sec:dataaug}). Integrating augmentation methods into attacks
via a flexible framework (\algoref{alg:dataaug}), we
searched for augmentation-combinations that can help boost
transferability via inefficient but optimal exhaustive search on a
small subset of augmentations and efficient, near-optimal genetic
search on a search space containing all augmentations
(\secref{sec:ultcombo}). We conducted extensive experiments using an
\imagenet{}-compatible dataset, \cifarten{}~\cite{krizhevsky09cifar},
and 18 models, %
and measured transferability in diverse settings, including with and
without defenses (\secrefs{sec:setup}{sec:results}). Furthermore, by
studying the impact of augmentations on model gradients, we offer a
theoretical explanation for why certain augmentations can promote
transferability (\secref{sec:theory}) that we later support with
empirical results (\secref{sec:res:theorysupport}).

In a nutshell, we make the following contributions:
\begin{itemize}[leftmargin=*]
  \itemsep0em
  \item We find that simple color-space augmentations exhibit
    performance commensurate with that of state-of-the-art
    techniques. Notably, they  surpass state-of-the-art techniques in
    terms of transferability to adversarial training models, while
    simultaneously offering a reduction in running time
    costs (\secref{sec:res:singleaugs} and \secref{sec:time}). 
  \item We propose parallel composition to integrate a large
    number of augmentations into attacks
    (\secref{sec:dataaug:compose}) and find it 
    boosts transferability compared to composition approaches
    previously used (\secref{sec:res:ParVsSer}). Further, with parallel
    composition, we 
    discover that transferability has a mostly monotonic relationship
    with the new data-augmentation techniques we introduce and ones
    used in prior work (\secref{sec:res:monotonic}). 
  \item We show that attacks integrating the best
    augmentation combinations discovered by exhaustive and genetic
    search (\bestcombo{} and \bestcomboadv{}, respectively) 
    outperform state-of-the-art attacks by a large margin
    (\secref{sec:res:bestcombo}). 
  \item We theoretically demonstrate that augmentations can smoothen
    surrogate models' gradients, which in turn can
    improve transferability (\secref{sec:theory}), and empirically
    back the theory (\secref{sec:res:theorysupport}).
\end{itemize}

\section{Background and Related Work}

\parheading{Evasion Attacks}
Many evasion attacks assume adversaries have white-box
access to models---i.e., adversaries know models' architectures and weights
(e.g.,~\cite{goodfellow2014explaining, szegedy2013intriguing,
  carlini2017towards}). These %
typically leverage first- or
second-order optimizations to generate \AE{}s models would
misclassify. For example, given an input $x$ of class $y$, model
weights $\theta$, and a loss function $J$, the Fast Gradient Sign method
(\fgsm{}) of~\cite{goodfellow2014explaining}, crafts an \AE{} $\hat{x}$ using the loss gradients $\nabla_x J(x,y,\theta)$:
$$\hat{x}=x+\epsilon*\operatorname{sign}(\nabla _{x}J(x,y,\theta))$$
where $\operatorname{sign}(\cdot)$ maps real numbers to -1, 0, or 1,
depending on their sign. Following \fgsm{}, researchers proposed
numerous advanced attacks. Notably, iterative \fgsm (\ifgsm{})
performs multiple update steps to $\hat{x}$ to evade models~\cite{kurakin2018adversarial}:
$$\quad \hat{x}_{t+1}=\operatorname{Proj}_x^\epsilon\biggl(\hat{x}_t + \alpha \cdot \operatorname{sign}\Big(\nabla_x J\left(\hat{x}_t, y, \theta\right)\Bigr)\biggr)$$
where $\operatorname{Proj}_x^\epsilon(\cdot)$ projects the perturbation
into \lpnorm{\infty}-norm $\epsilon$-ball centered at $x$, $\alpha$
is the step size, and $\hat{x}_0=x$. In this work, we study  %
attacks based on \ifgsm{}.

In practice, adversaries often lack white-box access to victim
models. Hence, researchers studied black-box attacks  in which
adversaries may only query models. Certain attack types, such as
score- and boundary-based attacks perform multiple queries, often
around several  thousands, to produce \AE{}s 
(e.g.,~\cite{brendel2017decision, ilyas2019prior}).
By contrast, attacks leveraging \emph{transferability}
(e.g.,~\cite{goodfellow2014explaining, szegedy2013intriguing}) avoid
querying victim models, and use surrogate white-box models to create
\AE{}s that are likely misclassified by other black-box ones.

Enhancing transferability is an active research area.
Some methods integrate momentum into attacks such as
\mifgsm{} to avoid surrogate-specific optima and saddle points that
may hinder transferability (e.g.,~\cite{dong2018boosting,
  wang2021vmi,wang2021boosting}). Others employ specialized 
losses, such as reducing the variance of intermediate
activations~\cite{huang2019enhancing} or the mean loss of model 
ensembles~\cite{liu2016delving}, to enhance transferability. Lastly,
a prominent family of attacks
leverages data augmentation to enhance \AE{}s' 
generalizability between models.
For instance, Dong et al.\ boosted transferability by
integrating random translations into \ifgsm{}~\cite{dong2019evading}.
Evasion attacks incorporating data augmentation attain
state-of-the-art transferability rates~\cite{lin2019nesterov,
  wang2021vmi, Wang21Admix}. Nonetheless, prior work has only considered a
restricted set of four augmentation methods for boosting
transferability (see \secref{sec:dataaug:prev}). By contrast, we %
investigate augmentations' role
at enhancing transferability more systematically, by
exploring how a more comprehensive set of augmentation types and their
compositions affect transferability.
Some efforts explored how augmentation methods used during training
affect transferability~\cite{Mao22Transfer,
  zhang2023does}. Unlike these, 
we study augmentations' role in boosting
transferability when incorporated into  attacks.

\parheading{Defenses}
Various defenses were proposed to mitigate evasion
attacks. Adversarial training---a procedure integrating correctly labeled
\AE{}s in training---is of the most
effective methods for enhancing adversarial robustness
(e.g.,~\cite{goodfellow2014explaining, tramer2017ensemble}). %
Other defense methods sanitize inputs prior to classification
(e.g.,~\cite{guo2017countering}); %
attempt to detect attacks (see~\cite{tramer2022detecting}); or seek
to certify robustness in $\epsilon$-balls around inputs
(e.g.,~\cite{cohen2019certified, salman2019provably}). Following 
standard practices in the literature~\cite{Wang21Admix}, we
evaluate transferability-based attacks against a representative set of
these defenses. %

\section{Data Augmentations for Transferability}
\label{sec:dataaug}

Data augmentation is traditionally used in training, to
reduce overfitting and improve generalizability~\cite{shorten2019survey}.
Inspired by this use, transfe\-rability-based attacks
adopted data augmentation
to limit overfitting to surrogate
models and produce
\AE{}s likely to generalize and be misclassified by victim
models. \algoref{alg:dataaug} presents a general framework for
integrating data augmentation into \ifgsm{} with momentum
(\mifgsm{})~\cite{dong2018boosting}.
In the framework, a method $D(\cdot)$ augments the attack
with $m$ variants of the estimated \AE{} at each
iteration. Consequently, the adversarial perturbation found by the
attack increases the expected loss over transformed counterparts of
the benign sample $x$ (i.e., the distribution set by $D(\cdot)$ given
$x$). %
\begin{algorithm}[t!]
\caption{\label{alg:dataaug}\mifgsm{} with augs.}
\begin{algorithmic}[1]
\STATE {\bfseries Input:} Sample $x$; label $y$;
  loss $J(\cdot)$; model params.\ $\theta$; \# iters.\ $T$;
    momentum $\mu$; norm-bound $\epsilon$;
    method $D(\cdot)$ producing $m$ augmented samples; step size $\alpha$.
    \STATE $\alpha=\left.\epsilon/T\right.$\;
    \STATE $\hat{x}_{0}=x$ 
    \STATE $g_0=0$ 
    \FOR{$t=0$ to $T-1$}
        \STATE$\bar{g}_{t+1}=\frac{1}{m} \sum_{i=0}^{m-1} \nabla_{x}\left(J\left(D(\hat{x}_t)_i, y, 
        \theta\right)\right)  $  
        \STATE $g_{t+1}=\mu \cdot g_t+\frac{\bar{g}_{t+1}}{\left\|\bar{g}_{t+1}\right\|_1}$ 
        \STATE$ \hat{x}_{t+1}=\operatorname{Proj}_x^\epsilon\big(\hat{x}_t+\alpha \cdot \operatorname{sign}\left(g_{t+1}\right)\big)$ 
    \ENDFOR
    \STATE \textbf{return} $\hat{x}=\hat{x}_T$
\end{algorithmic}
\end{algorithm}

The framework in \algoref{alg:dataaug} is flexible, and can admit any
augmentation method. We use it to describe previous attacks employing
augmentations and to systematically
explore new ones. Next, we detail previous attacks, describe
augmentation methods we adopt for the first time to enhance
transferability, and explain how these can be combined for best
performance.%

\subsection{Previous Attacks Leveraging Augmentations}
\label{sec:dataaug:prev}

Prior work explored the following augmentation methods (to set
$D(\cdot)$ in \algoref{alg:dataaug}).

\parheading{Translations}
Using random translations of inputs,
Dong et al.\ proposed a translation-invariant attack
to promote transferability~\cite{dong2019evading}.
They also offered an
optimization to reduce the attack's time and space complexity by
convolving the model's gradients
(w.r.t. non-translated inputs) with a Gaussian kernel.
While we use this
optimization in the implementation in the interest of efficiency, we
highlight that \algoref{alg:dataaug} can capture the attack.

\parheading{Diverse Inputs}
Xie et al.\ proposed a size-invariant attack. Their
augmentation randomly crops $\hat{x}_t$ and resizes
crops per the model's input dimensionality~\cite{xie2019improving}.

\parheading{Scaling Pixels}
Lin et al.\ showed that adversarial perturbations invariant
to scaling pixel values transfer with higher success between deep
neural networks (\dnn{}s)~\cite{lin2019nesterov}.
In their case,
$D(\cdot)$ produces $m$ samples such that
$D(x)_i = \frac{x}{2^i}$ for $i\in\{0, 1, ..., m-1\}$, where $m$=5 by default.

\parheading{Admix}
Wang et al.\ assumed that the adversary has a gallery of images
from different classes and adopted augmentations similar to
\emph{MixUp}~\cite{Wang21Admix, zhang2017mixup}. For each sample
$x^{\prime}$ from the gallery (typically selected at random from the batch), Admix augments attacks with $m$
(typically set to 5) samples, such that
$D(x, x^{\prime})_{i}=\frac{1}{2^i} \cdot{}(\hat{x}_t+\eta\cdot{}x^{\prime})$,
where $i\in\{0, 1, ..., m-1\}$, and $\eta\in [0, 1]$ is
set to 0.2 by default. Notably, Admix degenerates
to {\it scaling pixels} when $\eta=0$.

\parheading{Uniform Noise (UN)} Li et al.\ added uniform noise to samples to
produce uniform-noise-invariant attacks with enhanced 
transferability~\cite{li2023towards}. 

\parheading{Dropping Patches (DP)}
Li et al.\ and Wang et al.\ divided perturbations into (16$\times$16)
grids and randomly sampled parts of the grid to
drop~\cite{li2023towards, wang2020unified}.

\smallskip{}
Leading transferability-based attacks compose \emph{(1)}
diverse inputs, scaling, and translations
(Lin et al.'s \dst{} attack~\cite{lin2019nesterov},
and Wang et al.'s \sdtv{} attack that
also tunes the gradients' variance~\cite{wang2021vmi});
or \emph{(2)} \admix{}, diverse inputs, and translation
(Wang et al.'s \admixdt{} attack~\cite{Wang21Admix}); or
\emph{(3)} uniform noise, dropping patches, diverse inputs, and
translations (Li et at.'s \undp{}~\cite{li2023towards}). We consider
these attacks as baselines in our experiments (\secrefs{sec:setup}{sec:results}).

\subsection{New Augmentations for Boosting Transferability}
\label{sec:dataaug:new}

While prior work studied the effect of some data-augmentation methods'
effect
on transferability, a substantially wider range of
data-augmentation methods exist. Yet, the impact of these
on transferability remains unknown. To fill the gap,
we examined Shorten et al.'s
survey on data augmentation~\cite{shorten2019survey}
as well as two prominent image-augmentation libraries~\cite{imgaug,
  torchvision} 
for reducing overfitting in deep learning to identify a comprehensive
set of augmentation methods that may boost transferability.
Overall, we identified 46 representative methods of seven categories
to evaluate. 
We present a representative subset of %
augmentations per category below; \appref{app:otheraug} lists the
remaining augmentations. %

\parheading{Color-space Transformations}
Potentially the simplest augmentation types are those
applied in color-space. %
Given images represented as three-channel tensors, methods in this
category manipulate pixel values only based on information encoded in
the tensors. We evaluate 12 color-space transformations.
Among them, we consider \textit{color jitter (\colorjitter{})},
which applies random color manipulations~\cite{wu2015deep}. 
Specifically, we consider random adjustments of pixel values within a
pre-defined ranges of hue, contrast, saturation, and
brightness. %
Additionally, we consider
\textit{greyscale (\greyscale{})} augmentations.
This simple augmentation converts images into greyscale
(replicating it three times to obtain an RGB representation).
Mathematically, the conversion is calculated by
$\omega_R\cdot{}x^R+\omega_G\cdot{}x^G+\omega_B\cdot{}x^B$,
where $x^R$, $x^G$, and $x^B$, correspond to the RGB channels, 
respectively, and $\omega_R$, $\omega_G$, and $\omega_B$, all
$\in [0, 1]$, denote the channel weights, and sum up to 1. 

\parheading{Random Deletion}
Different from DP, which drops random regions from
\emph{perturbations}, we consider random deletions of \emph{input samples'} regions, replacing
them with others values. For example,
inspired by dropout regularization, \textit{random erasing
    (\randerase{})} helps \ml{} models focus on descriptive
features of images and promote robustness to
occlusions~\cite{zhong2020random}. To do so, randomly 
selected rectangular regions in images are replaced by masks composed of
random pixel values. Similarly to \randerase{},
\textit{\cutout{}} masks out regions of inputs to
improve \dnn{}s' accuracy~\cite{devries2017improved}. The main
difference from \re{} is that \cutout{} uses fixed masking values,
and may perform less aggressive masking when selected regions lie
outside the image.

\parheading{Kernel Filters}
Convolving images with filters of different types can produce certain
effects, such as blurring (via Gaussian kernels), sharpening (via
edge filters), or edge enhancement. We study the effect of 13 %
filters on transferability, including edge-enhancement via
\textit{sharpening (\sharpen{})}.

\parheading{Mixing Images}
Some augmentation methods (e.g., Admix) fuse
images together. %
We consider \textit{\cutmix{}}, which replaces a region within one
image with a region from another image picked from a
gallery~\cite{yun2019cutmix}. 

\parheading{Style Transfer}
Certain augmentations change the image style. We study six of these
in this work. Among the six, we also consider augmentations using
\textit{neural transfer (\neuraltrans{})}  
to preserve image semantics while changing their style. Specifically, we
use~\cite{gatys2015neural}'s  generative model to transfer image
styles to that of Picasso's 1907 self-portrait.

\parheading{Meta-learning-inspired Augmentations}
Meta-learning is a subfield of \ml{} studying how \ml{}
algorithms can optimize other learning algorithms~\cite{hospedales2021meta}. 
In the context of data augmentation, algorithms such as
\textit{\autoaug{}} have been proposed to train controllers
to select an appropriate augmentation method to avoid
overfitting~\cite{cubuk2018autoaugment}. We use the pre-trained
\autoaug{} controller, encoded as a recurrent neural network, to
select augmentation methods and their magnitude from a set of
pre-defined augmentation methods.

\parheading{Spatial Transformations}
Augmentations performing spatial transformations alter the locations
of pixels in the image. Random translation and resizing
(\secref{sec:dataaug:prev}) belong to this category. We consider seven
additional augmentations that perform spatial transformation,
including random rotations, affine transformations, and horizontal
flipping.

\subsection{Composing Augmentations}
\label{sec:dataaug:compose}

There exist two ways to
compose data-augmentation methods in attacks,
namely: \emph{parallel} and \emph{serial} composition. 
In parallel composition, each augmentation method is applied independently
on the input, and their outputs are aggregated by taking their union
as $D(\cdot)$'s output to augment attacks. By contrast, serial composition
applies augmentations sequentially, one after the other, where the
first method operates on the original sample, and each subsequent
augmentation operates on its predecessor's outputs. Consequently,
serial composition leads to an {\em exponential} growth in the number of samples,
while parallel composition leads to a {\em linear} growth. 
All baselines we consider in this work~\cite{lin2019nesterov,
  wang2021vmi, Wang21Admix, li2023towards} %
employ serial composition, as this is
the only composition method previously used. Compared to prior work, %
we consider a substantially larger number of augmentation methods, which
may lead to prohibitive memory and time requirements in the case of serial
composition. Additionally, because the order of applying certain
augmentations matters (e.g., \greyscale{} then \cutmix{} leads
to different outcome that \cutmix{} followed by \greyscale{}),
exploring a meaningful number of serial compositions (e.g., out of
$\ge$46$!$ possibilities) becomes virtually impossible. Moreover,
serial composition of images may distort images drastically, thus
degrading transferability, as we find (\secref{sec:ParVsSer}). 
Accordingly, we mainly consider parallel composition between data-augmentation methods.
We \emph{only} serially compose translations, scaling, and diverse inputs, for
consistency with prior work (e.g.,~\cite{Wang21Admix}).

\subsection{Discovering Effective Compositions}
\label{sec:ultcombo}

We use two approaches to discover augmentation compositions that
promote transferability. As a simple, yet effective method, we
employ exhaustive search on a small set of augmentations. We note that,
while Li et al.\ considered a similar approach~\cite{li2023towards},
they only \emph{(1)} evaluated augmentations previously established
as conducive transferability, as opposed the the new augmentations we
incorporate; and \emph{(2)} tested serial compositions of
augmentations, which we find to perform worse than parallel
composition %
(\secref{sec:res:ParVsSer}). %
Exhaustive search enabled us to uncover a roughly monotonic relationship between augmentations and
transferability (\secref{sec:res:monotonic}).
Additionally, we apply genetic search over the full set of
augmentations. Here,
informed by the roughly monotonic relationship between augmentations and
transferability, we bias the search hyperparameters toward including
augmentations within combinations to speed up convergence toward
combinations boosting transferability.

\parheading{Exhaustive Search}
\label{sec:ultcombo:exhaust}
A straightforward approach to find a set of augmentations that
promote transferability when combined is to exhaustively evaluate all
possible combinations. On the one hand, this approach is guaranteed to
find the combination that advances transferability the most, and
enables us to precisely characterize the relationship between
augmentations and transferability---e.g., whether transferability is
monotonic in the number of augmentations included. However, it would
be infeasible to exhaustively test all %
possible combinations included in the power set of all augmentations
we consider in this work. Hence, as a middle ground, we evaluate
exhaustive search with a representative, well-performing augmentation
method per data-augmentation category.

\parheading{Genetic Search}
\label{sec:ultcombo:gen}
As an efficient alternative for exhaustive search, we employ genetic
search---a commonly used optimization technique suitable for large,
discrete search spaces~\cite{katoch2021review}. Genetic search algorithms seek to
simulate biological evolution processes to find near-optimal
solutions. Instead of evaluating the entire search space, genetic
search begins with an initial population sampled from the search
space. Subsequently, the algorithm selects the fittest individuals
from the population and applies crossover and mutation actions between
them to produce a generation that should score higher on the fitness
objective one aims to optimize. Once the fitness value of the
best-performing individuals converges or a maximum number of iteration is 
reached, the search ends and the best performing individual is
returned.

In this work, we seek to search over a population of
augmentation combinations to maximize transferability (i.e., the
fitness function). To this end, starting with a population of
containing several combinations, we evaluate the transferability rate
achievable via each combination. Next, to produce the next generation,
we select the fittest combinations with the highest transferability
rates at random, with replacement, to act as ancestors, while
maintaining the population size.
The probability for selecting each combination is equal to the ratio
between its fitness and the total fitness of all combinations in the
population. Then, for each selected combination, with probability
$p_\mathit{cross}$ we perform crossover with another combination
chosen at random from combinations that passed selection. If no
crossover occurs, we pass the combination as is to the next
generation after performing mutation.
In crossover, we select (resp. exclude) an
augmentation if it is included (resp.\ excluded) in both ancestors, and
select whether to include the augmentation at random when it is
included in one ancestor but excluded in the other. Subsequently, we
perform mutations, removing (resp.\ inserting) and augmentation that
is included (resp.\ removed) with probability $p_\mathit{mutate}$. We
stop the process after $n_\mathit{gen}$ search iterations.

As exhaustive search demonstrated a roughly monotonic relationship
between augmentations and transferability
(\secref{sec:res:monotonic}), we biased the genetic search
toward including augmentations to speed up convergence to
highly-performing combinations. More precisely, we selected an initial
population with probability of $p_\mathit{aug}>$50\% to include each
augmentation in a combination. Thus, augmentation combinations
in the initial and subsequent populations had higher likelihood to
include augmentations than exclude them.

Genetic search is not guaranteed to converge to the optimal
combinations as long as we have not covered the entire search space.
Still, we find that, for small search spaces, genetic search quickly
discovers combinations that attain roughly as high transferability rates
as the best combination found by exhaustive search (\secref{sec:res:bestcombo}).
After identifying the genetic-search
hyperparameters that work well for small search spaces, we apply
genetic search to the large search space containing all augmentations.
Doing so enables us to identify a combination of augmentations 
attaining state-of-the-art transferability rates when composed
(\secref{sec:res:gensearch}).

\section{Theoretical Analysis of Augmentations}
\label{sec:theory}

Before turning to empirical evaluations, we theoretically
analyze why augmentations can boost transferability.
Specifically, we explore how augmentations affect the gradients
attacks leverage to improve our understanding of how they help
\AE{}s generalize from surrogates to targets. Our analysis
shows that certain augmentations smoothen the model gradients
(i.e., lead gradients to be more consistent across attack
iterations). Intuitively, in return, this may limit the effect of
surrogate peculiarities (e.g., surrogate-specific %
changes in the classification boundaries) on the attack, and increase
\AE{}s' generalization. 
By contrast, when gradients are
sensitive to surrogates' peculiarities, they are likely to change
drastically within small regions of input domains (e.g., across
attack iterations), hindering transferability to target models. 
Our experiments (\secref{sec:res:theorysupport}) empirically support the
theoretical analysis and intuition. %

Adapting proof techniques from randomized
smoothing~\cite{cohen2019certified, salman2019provably}, we
show that Gaussian noise helps smoothen
gradients. Subsequently, we extend the result to a more
general set of augmentation methods, showing that gradients
can be smoothened by augmenting attacks with additive noise sampled
from smooth distributions---i.e., distributions that do not exhibit
sharp changes in their density within small regions. %

\parheading{Gaussian-Noise Augmentations}
\label{sec:theory:gauss}
Prior work has shown that augmenting inference with Gaussian noise, a
technique known as randomized smoothing, leads to estimating a smooth
function with a bounded Lipschitz constant (i.e., maximum rate of
change in the region the function is
defined)~\cite{cohen2019certified, salman2019provably}. Differently
from prior work, instead of augmenting inference, ours augments the
derivative (i.e., gradient estimates). We demonstrate that when using
Gaussian noise for data augmentation, the derivative becomes
smooth. More formally, we show that the derivative's Lipschitz
constant becomes bounded. Moreover, we demonstrate that the
smoother, less peaky, the Gaussian distribution becomes, the smaller is
the derivative's Lipschitz constant and the smoother is the
derivative.

For a sample and label pair $(x,y)$, we denote the loss derivative as
$f(x)=\nabla_x J\left(x, y, \theta\right)$.
Let $I$
be the Lipschitz constant of $J\left(x, y, \theta\right)$
(i.e., the surrogate's loss). Note that $I$ is typically bounded for
\dnn{}s and can be approximated via  numerical analysis
techniques~\cite{virmaux2018lipschitz}.
Our first theoretical result shows that $\hat{f}$, a variant of $f$
augmented with noise sampled from a zero-mean Gaussian distribution
with an identity covariance matrix $E$ (i.e., noise drawn from
$\mathcal{N}(0, E)$), is smooth, with a Lipschitz constant of
$I\cdot\sqrt{\frac{2}{\pi}}$.

\begin{theorem}
 \label{theoA}
 $\hat{f}$, attained by augmenting
 $\nabla_x J\left(x, y, \theta\right)$
  with noise drawn from $\mathcal{N}(0, E)$,  is
  $I\cdot\frac{\sqrt{2}}{\sqrt{\pi}}$-Lipschitz, where $I$
  is the Lipschitz constant of $J\left(x, y, \theta\right)$.
\end{theorem}

\begin{proof}
  Augmenting $f$ with Gaussian noise, drawn from $\mathcal{N}(0, E)$,
  results in the so-called Weierstrass transform of $f$~\cite{salman2019provably}:
  
$$\begin{gathered}
\hat{f}(x)=(f * \mathcal{N}(0, E))(x) \\
=\frac{1}{(2 \pi)^{n / 2}} \int_{{R}^n} f(t) \exp \left(-\frac{1}{2}\|x-t\|^2\right) d t \\
=\int_{{R}^n} f(t) \frac{1}{(2 \pi)^{n / 2}} \exp \left(-\frac{1}{2}\|x-t\|^2\right) d t \\
=\int_{{R}^n} \nabla_x\left(J\left(t, y, \theta\right)\right) \frac{1}{(2 \pi)^{n / 2}} \exp \left(-\frac{1}{2}\|x-t\|^2\right) d t
\end{gathered}$$
  where $t$ is obtained after adding noise to $x$.

It suffices to show %
$u\cdot \nabla \hat{f}(x)\le I\cdot\sqrt{\frac{2}{\pi}}$.
holds for any unit vector $u$.
Note that:
$$\begin{gathered}
\nabla \hat{f}(x)  =
\frac{1}{(2 \pi)^{n / 2}} \int_{{R}^n} f(t)(x-t)
\exp \left(-\frac{1}{2}\|x-t\|^2\right) d t.
\end{gathered}$$    
Thus, using $|f(t)|\le I$, we get:
$$\begin{gathered}
u \!\cdot\! \nabla \hat{f}(x) \!\leq \! \frac{1}{(2 \pi)^{n / 2}}\! \int_{{R}^n} \! \!\!\!
\mid \! \! I  \! \cdot  \! u \! \cdot\! (x-t) \mid \exp \left(-\frac{1}{2}\|x-t\|^2\right) d t \\
 =\frac{I}{\sqrt{2 \pi}} \int_{-\infty}^{+\infty}|s| \exp \left(-\frac{1}{2} s^2\right) d s=I\cdot\sqrt{\frac{2}{\pi}} .
\end{gathered}$$
\end{proof}

This theorem shows
that augmenting noise sampled from $\mathcal{N}(0, E)$ bounds the
gradients' maximum rate of change, making them smooth. The next
theorem generalizes the result, showing that augmenting noise sampled
from an isotropic Gaussian distribution, $\mathcal{N}(0, \sigma{}E)$,
leads to smoother gradients the larger is $\sigma{}$. Said
differently, when the distribution is less peaky, the Lipschitz
constant of the $\hat{f}$, the augmented variant of the derivative,
$f$, becomes smaller.

\begin{theorem}
  \label{theoB}
  $\hat{f}$, attained by augmenting
  $\nabla_x J\left(x, y, \theta\right)$
  with noise drawn from
  isotropic Gaussian distribution, $\mathcal{N}(0, \sigma{}E)$,  is
  $I\cdot\frac{\sqrt{2}}{\sqrt{\pi}{\sigma}^2}$-Lipschitz, where $I$
  is the Lipschitz constant of $J\left(x, y, \theta\right)$.
\end{theorem}

\begin{proof}
  It suffices to show %
$u\cdot \nabla \hat{f}(x)\le I\cdot\frac{\sqrt{2}}{\sqrt{\pi}*{\sigma}^2}$
holds for any unit vector $u$.
Note that:
$$\begin{gathered}
\hat{f}(x)=\left(f * \mathcal{N}\left(0, \sigma^2\right)\right)(x)\\
=\frac{1}{(2 \pi)^{n / 2} * \sigma} \int_{R^n} f(t) \exp \left(-\frac{1}{2}\|x-t\|^2 / \sigma^2\right) d t \\
=\int_{R^n} f(t) \frac{1}{(2 \pi)^{n / 2} / \sigma} \exp \left(-\frac{1}{2}\|x-t\|^2 / \sigma^2\right) d t\\
=\int_{R^n} \nabla_x(J(t, y, \theta)) \frac{1}{(2 \pi)^{n / 2} * \sigma} \exp \left(-\frac{1}{2}\|x-t\|^2 / \sigma^2\right) d t.\\
\end{gathered}$$
Thus, if we compute the derivative of $\hat{f}(x)$, we would get:
$$\begin{gathered}
\nabla \hat{f}(x)=
\frac{1}{(2 \pi)^{n / 2} * \sigma^3} \int_{R^n} f(t)(x-t) \exp \left(-\frac{1}{2}\|x-t\|^2 / \sigma^2\right) d t
\end{gathered}$$
Now, using $|f(t)|\le I$, we get:
$$\begin{gathered}
u \cdot \nabla \hat{f}(x) \leq \frac{1}{(2 \pi)^{n / 2} * \sigma^3} \int_{R^n} \!\! |I \cdot u \cdot(x-t)| \exp \left(-\frac{1}{2}\|x-t\|^2 / \sigma^2\right) \! d t \\
=\frac{I}{\sqrt{2 \pi} * \sigma^3} \int_{-\infty}^{+\infty}|s| \exp \left(-\frac{1}{2} s^2 / \sigma^2\right) d s \\
=\frac{2 I}{\sqrt{2 \pi} * \sigma^3} \int_0^{+\infty} s \exp \left(-\frac{1}{2} s^2 / \sigma^2\right) d s \\
=\frac{2 I}{\sqrt{2 \pi} * \sigma^3} \int_0^{+\infty} \frac{1}{2} \exp \left(-\frac{1}{2} s^2 / \sigma^2\right) d s^2 \\
=\frac{2 I}{\sqrt{2 \pi} * \sigma^3} \int_0^{+\infty} \frac{1}{2} \exp \left(-\frac{1}{2} x / \sigma^2\right) d x \quad\left(\text{where}\ x=s^2\right) \\
=\left.\frac{2 I}{\sqrt{2 \pi} * \sigma^3} *\left(-\sigma * \exp \left(-\frac{1}{2} x / \sigma^2\right)\right)\right|_0 ^{\infty}
=I\cdot\frac{\sqrt{2}}{\sqrt{\pi}*{\sigma}^2}.
\end{gathered}$$
\end{proof}

This theorem demonstrates that, when augmenting noise sampled from an 
isotropic Gaussian distribution, we can smoothen the derivative,
decreasing gradients' rate of change, by increasing $\sigma$. In return,
the gradient estimates would become more consistent across attack
iterations. This, we expect, renders the surrogate's peculiarities
less likely to impact \AE{}s, increasing the likelihood of
generalization to the target.

\parheading{Additive Noise from Smooth Distributions}
We extend the theoretical results, %
showing that we can smoothen  $f$ if we augment inputs
with additive noise sampled from a smooth distribution
$g(\cdot)$.
Intuitively, a distribution $g(\cdot)$ is smooth if
it does not exhibit sharp changes in its density within small
regions. 
Formally, for a constant $A$, we define a distribution $g$ as smooth
if $|\int_{{R}^n} \nabla_x g(t-x) d t |\le A$, where $t$ is
obtained by adding noise sampled from $g$ to $x$. For example, for
an isotropic Gaussian distribution, it can be shown that
$A=\frac{\sqrt{2}}{\sqrt{\pi}\sigma^2}$.
Given this definition, we generalize the previous %
theoretical results, and show that the 
Lipschitz constant of $\hat{f}$, obtained by adding noise sampled from
$g$ to $f$, is bounded.

\begin{theorem}
  \label{theoC}
  $\hat{f}$, attained by augmenting
  $\nabla_x J\left(x, y, \theta\right)$
  with noise drawn from a distribution $g$, such that
  $|\int_{{R}^n} \nabla_x g(t-x) d t |\le A$,
  is $I\cdot A$-Lipschitz, where $I$
  is the Lipschitz constant of $J\left(x, y, \theta\right)$.
\end{theorem}

\begin{proof}
  Again, it suffices to show
$u \cdot \nabla \hat{f}(x) \leq I \cdot A.$
holds for any unit-norm vector $u$. Note that:
$$\begin{gathered}
  \hat{f}(x)=f * g(x)=\int_{{R}^n} f(t) g(t-x) d t
\end{gathered}$$
Thus:
$$\begin{gathered}
u \cdot \nabla \hat{f}(x) =\int_{{R}^n}  f(t) \cdot u \cdot \nabla_t g(t-x) d t \\
\leq|\int_{{R}^n}  f(t) \cdot u \cdot \nabla_t g(t-x) d t|
\leq\int_{{R}^n} | f(t) \cdot u| \cdot |\nabla_t g(t-x) |d t\\
\leq \int_{{R}^n} I\cdot |\nabla_t g(t-x) |d t
\leq I |\int_{{R}^n} \nabla_x g(t-x) d t|=I \cdot A
\end{gathered}
$$
\end{proof}

This theorem demonstrates that there exist data augmentations other
than Gaussian noise that can enable us to obtain a smooth
$\hat{f}$. While we cannot ascertain that all augmentations we use
satisfy the pre-conditions necessary for the theorem to hold (i.e.,
smooth distribution $g$), we find that, in practice, the augmentations
we use lead to smoother gradient estimates
(\secref{sec:res:theorysupport}). In return, we also find that these
augmentations help boost transferability.

\section{Experimental Setup}
\label{sec:setup}

Now we present the experimental setup.
Our code is publicly
available~\cite{OurCode}.

\parheading{Data}
We used an \imagenet{}-compatible
dataset~\cite{imagenetcompat} %
and \cifarten{}~\cite{krizhevsky09cifar} for evaluation, per common practice
(e.g.,~\cite{dong2019evading, yang2021trs}).
The former contains 1K 224$\times$224 dimensional images
pertaining to 1K classes, originally collected for the
NeurIPS 2017 adversarial \ml{} competition.
For the latter, we sampled 1K 32$\times$32
dimensional test-set images, roughly balanced between the dataset's
ten classes. As the findings across datasets are consistent,
we present detailed results primarily on \imagenet{}.

\parheading{Models}
We used 18 \dnn{}s to transfer attacks from (as surrogates) and
to (as targets)---six for \cifarten{} and 12 for \imagenet{}.
To facilitate comparison with prior work, we included models
widely used for testing transferability (e.g.,~\cite{Wang21Admix, yang2021trs}).
Furthermore, to ensure our findings are general, we included
models covering varied architectures, including Inception, ResNet, VGG,
DenseNet, MobileNet, ViT, and NASNet.
Of the 12 \imagenet{} models, eight were normally
trained and four were adversarially
trained. Specifically, for normally trained models, we selected: 
Inc-v3~\cite{szegedy2016rethinking};
Inc-v4;
IncRes-v2~\cite{szegedy2017inception});
Res-50;
Res-101;
Res-152~\cite{he2016identity};
MNAS~\cite{Tan_2019_CVPR};
and %
ViT~\cite{50650}.
For adversarially trained models, we
used:
Inc-v3$_\mathit{adv}$~\cite{kurakin2016adversarial};
Inc-v3$_\mathit{ens3}$;
Inc-v3$_\mathit{ens4}$;
and %
IncRes-v2$_{ens}$~\cite{tramer2017ensemble}. 
All six \cifarten{} \dnn{}s were normally
trained. For this dataset, we used pre-trained %
VGG~\cite{vgg};
Res~\cite{he2016identity};
DenseNet~\cite{huang2017densely};
MobileNet~\cite{sandler2018mobilenetv2};
GoogleNet~\cite{szegedy2015going};
and
Inc~\cite{szegedy2016rethinking}.
We obtained the models' PyTorch implementations and
weights from public GitHub
repositories~\cite{cifarmodels, modelrepo, torchvision}.

\parheading{Attacks}
We tested standard attack configurations
and validated findings with varied perturbation norms,
similar to prior work~\cite{Wang21Admix, yang2021trs}. Namely, we
evaluated untargeted \mifgsm{}-based attacks, bounded 
in \lpnorm{\infty}-norm. For \imagenet{}, unless stated otherwise, we tested
$\epsilon=\frac{16}{255}$, but also experimented with
$\epsilon\in\{\frac{8}{255}, \frac{24}{255}\}$ (\appref{app:imagenet}). For \cifarten{}, 
we experimented with $\epsilon\in\{0.02, 0.04\}$.
We quantified attack success
via \emph{transferability rates}---the percentages of
attempts at which \AE{}s created against surrogates
were misclassified by targets.
As baselines, we used four state-of-the-art transferability-based 
attacks: \admixdt{}, \dst{}, \sdtv{}, and \undp{} (see \secref{sec:dataaug:prev}).
\appref{app:augparams} reports the parameters
used in attacks and augmentation methods.
Besides the four state-of-the-art baselines we tested, 
we considered including other recent attacks in the
evaluation~\cite{huang2019enhancing, Wu21AdvTforms, Zhang23PAM}.
However, these either lacked publicly available
implementations~\cite{Wu21AdvTforms, Zhang23PAM}, or achieved
uncompetitive transferability rates in our %
experiments~\cite{huang2019enhancing}.

\section{Experimental Results}
\label{sec:results}

We start by evaluating
individual augmentation methods and standard combinations with
scaling, diverse inputs, and translations
(\secref{sec:res:singleaugs}). We then turn to analyzing \emph{all possible}
compositions between a representative subset of augmentation types via
exhaustive search to assess whether
transferability typically improves when considering more
augmentations (\secref{sec:res:monotonic}). %
Next, we present the results of genetic search (\secref{sec:res:gensearch}).
In \secref{sec:res:bestcombo}, we report evaluate
\bestcombo{}---the best combination found by exhaustive 
search---and \bestcomboadv{}---the best combination found by genetic
search---including on defended models, and provide comparisons with
the baselines attacks. We then compare parallel and serial
composition, showing the former improves transferability
(\secref{sec:res:ParVsSer}), before turning to
attack run-times (\secref{sec:time}).
Last, we provide empirical support to the
theoretical results presented in \secref{sec:theory}
(\secref{sec:res:theorysupport}).

\subsection{Individual Augmentations}
\label{sec:res:singleaugs}

\begin{table*}[t!]
\centering
\caption{\label{tab:strongsingle}Avg.\ transferability (\%)
  on \imagenet{} from Inc-v3 to
  all other models when integrating individual
  augmentations composed with \dstonly{}.
  A vertical line separates the baselines.
  from our attacks.
}

\begin{tabular}{r|rrrrrrr}
\toprule
  \textbf{Attack}
  &\ASDTM{} &
  \dst{}&
  \sdtv{}&
  \multicolumn{1}{r|}{\undp{}}&
  \chanshuffle{}-\dstonly{} &
  \colorjitter{}-\dstonly{}&
  \fancyPCA{}-\dstonly{} \\

  \textbf{Avg.}
  &80.6
  &81.8
  &\textbf{88.8}
  &\multicolumn{1}{r|}{79.8}
  &84.8
  &85.9
  &{83.5} \\

  \midrule
  \textbf{Attack (cont.)} &
  \gsdt{} &
  \randerase{}-\dstonly{}&
  \cutmix{}-\dstonly{}&
  \cutout{}-\dstonly{}&
  \neuraltrans{}-\dstonly{}&
  \sharpen{}-\dstonly{}&
  \autoaug{}-\dstonly{}\\

  \textbf{Avg.\ (cont.)}
  &\textit{87.0}
  &84.9
  &54.0
  &84.1
  &73.5
  &80.1
  &82.9\\
  
\bottomrule
\end{tabular}
\end{table*}

Initially, we evaluated transferability integrating a single 
augmentation at a time in attacks, or when composing individual
augmentations with leading augmentations
proposed in prior
work~\cite{lin2019nesterov, Wang21Admix,li2023towards}---namely,
diverse inputs,
scaling, and translation (\dstonly{}). To this end, we selected ten 
augmentations to represent augmentations categories
(\tabref{tab:strongsingle}). We found that considering each of the ten
augmentations individually does \emph{not} lead to competitive
performance with the baselines.
However, composing individual augmentations with \dstonly{} enhanced
transferability markedly.
\tabref{tab:strongsingle} shows transferability rates on \imagenet{}.
Surprisingly, simple augmentations in color-space
fared particularly
well, outperforming most baselines and all advanced augmentation methods
(e.g., \autoaug{}) in most cases. Particularly, 
Composing \greyscale{} with \dstonly{} 
(\gdst{} attack) performed best in this setting,
outperforming all baselines but \sdtv{}. \gsdt{} attained strong
results when target models were either normally or adversarially
trained (\twotabrefs{tab:normal}{tab:ens2adv}),
with different $\epsilon$s, and on \cifarten{}.

\subsection{Exhaustive Search}
\label{sec:results:Mono}
\label{sec:res:monotonic}

We wanted to evaluate whether transferability is monotonic in the
number of augmentation types---i.e., whether composing more
techniques increases or, at least, preserves
transferability. To this end, we ran exhaustive search
(\secref{sec:ultcombo:exhaust}) over augmentations covering all
categories and assessed the effect of composing more augmentations on
transferability. Note, however, that running exhaustive search with
the 46 augmentations we consider as well as \dst{} and \admix{} (i.e.,
a total of 48 augmentation methods) would be prohibitive. Thus, we
selected the best performing augmentation method of each of the seven
categories (\tabref{tab:strongsingle})
and \dst{}, and
evaluated all 2\textsuperscript{7} (=128) possible compositions
(per~\secref{sec:dataaug:compose}). Specifically, we tested every
possible combination of \greyscale{}, \cutout{}, \sharpen{}, \neuraltrans{},
\autoaug{}, \admix{}, and \dst{}. Given a composition, we produced
\AE{}s against the Inc-v3 \imagenet{} \dnn{}, %
and computed the expected transferability rate against remaining
\imagenet{} \dnn{}s. %
Then, for every pair of attacks differing only in whether a
single augmentation method was incorporated in the composition, we
tested whether adding the augmentation method improved
transferability.

The results showed a mostly monotonic relationship between
transferability and augmentations. Except for \neuraltrans{} and
\sharpen{}, which sometimes harmed transferability, %
composing more augmentation methods increased or 
preserved transferability. I.e., with a few exceptions,
augmentation  compositions containing a superset of augmentations
compared to other compositions  
typically had larger or equal transferability.
Notably, comparing all compositions enabled us to find that a composition of
all seven augmentation methods except for \neuraltrans{} attained the best
transferability. We call this composition \bestcombo{}.

\subsection{Genetic Search}
\label{sec:res:gensearch}

While we considered a more comprehensive set of augmentations than
prior work to construct \bestcombo{}, this set was still relatively
restricted. To discover combinations that advance transferability
further while taking all augmentations into account, we ran genetic
search (\secref{sec:ultcombo:gen}). Initially, to set the genetic
search hyperparameters and find how close to optimal are the
combinations found by genetic search, we ran it on the set of seven
augmentations tested with exhaustive search. Specifically, we varied
the number of generations ($n_\mathit{gen}$) and population sizes in genetic search and
measured \emph{(1)} the fraction of the search space that would be
covered by genetic search compared to exhaustive search; and
\emph{(2)} the ratio between the average transferability rates
attained by the fittest combination found by genetic search and the
optimal combination included in \bestcombo{}. Our aim was to find
whether genetic search is capable of discovering augmentation
combinations that achieve nearly as high transferability as exhaustive
search (i.e., maximizing the second metric) while searching a small
fraction of the search space (i.e., minimizing the second metric). We
repeated the experiment 100 times, using Inc-v3 as a surrogate and
the remaining \imagenet{} models as targets. Because we tested all
augmentation combinations in exhaustive search, we did not need to
reproduce \AE{}s when running genetic search. After
hyperparameter sweep, we set $p_\mathit{cross}$=60\%,
$p_\mathit{mutate}$=10\%, and $p_\mathit{aug}$=55\% for best performance.

The results showed that with as litts as $n_\mathit{gen}$=2 and
population size of 20, it is possible to find compositions attaining
$>$99\% of \bestcombo{}'s transferability rates.
Accordingly, using these hyperparameters, we ran genetic search on all 46
augmentations we considered for the first time to boost
transferability, as well as \dst{} and \admix{}, for a total of 48
augmentation techniques. Again, we selected Inc-v3 as a surrogate and
the remaining \imagenet{} models as targets. The search process found
a combination of 33 augmentations that achieved the highest
transferability
(see \appref{app:bestcomboadv-augs} for the complete list).
We refer to the resulting attack composing all 33 combinations by
\bestcomboadv{}.

\subsection{The Most Effective Combinations}
\label{sec:res:bestcombo}

We evaluated the transferability of \AE{}s produced by composing the
augmentations found via exhaustive and genetic search (\bestcombo{} and
\bestcomboadv{}, respectively) extensively, testing
transferability to normally and adversarially trained
\dnn{}s, various defenses, and commercial systems. Additionally, we
also explored the role of the number of augmented images (sample size)
and augmentation methods for boosting transferability.

\parheading{Normally and Adversarially Trained Targets}
The \bestcombo{} obtained higher transferability to normally trained
models than the baselines (89.6\% vs.\ $\le$85.0\%
avg.\ transferability; %
\tabref{tab:normal}). This held across different values
of $\epsilon$  (\tabref{tab:app:imageneteps} in \appref{app:imagenet}),
and on the \cifarten{}
dataset with different architectures.
\bestcomboadv{}, composing more augmentations, reached even higher
transferability rates, with 92.6\%
average transferability rate %
to %
normally trained models on \imagenet{}. %

\begin{table*}[t!]
    \caption{\label{tab:normal} Average transferability (\%) of black-box attacks on 
    \imagenet{}, from all surrogates %
    to normally and adversarially trained targets. A 
    vertical line separates the baselines from our attacks.}
  \centerline{%
  \begin{tabular}{l |rrrr|rrr}
    \toprule
    \textbf{Targets}  & \textbf{\ASDTM{} }& \textbf{\SDTM{}} & \textbf{\sdtv{}} & \textbf{\undp{}} & \textbf{\gsdt{}}&\textbf{\bestcombo{}} &  \textbf{\advancedultimatecombo{}} \\
    \midrule
      \textbf{Normally trained} &79.7  & 82.1 & 82.9 & 85.0 & 86.2 & \textit{89.6} & \textbf{92.6} \\
      \textbf{Adversarially trained} &78.2  & 77.4 & 82.5 & 65.4 & 83.1 & \textit{85.3} & \textbf{91.8}\\
               
    \bottomrule
  \end{tabular}
  } %

\end{table*}

\begin{table*}[!t]
\centering
\caption{\label{tab:ens2adv}Transferability (\%) on \imagenet{},
  from an ensemble of normally trained surrogates (incl.\ Inc-v4,
  Res-50, Res-101 and Res-152) to adversarially trained targets. A 
    vertical line separates the baselines from our attacks. 
} 
\centerline{%
\begin{tabular}{l | r r r r | r r r}
  \toprule
  \textbf{Model} & \textbf{\ASDTM{}}& \textbf{\SDTM{}} &
    \textbf{\sdtv{}}& \textbf{\undp{}}& \textbf{\gsdt{}}& 
    \textbf{\bestcombo{}}& \textbf{\advancedultimatecombo{}} \\ \midrule
 
  \textbf{Inc-v3$_{adv}$} & 89.3 & 89.1 & 92.8 & 88.2 & 92.6 & \textit{93.2} & \textbf{96.2}\\
  \textbf{Inc-v3$_{ens3}$} & 90.0 & 89.6  & 93.8 & 90.6 & 93.5 & \textit{95.4} & \textbf{96.3} \\
  \textbf{Inc-v3$_{ens4}$} & 89.0 & 87.8 & 93.1 & 86.1 & 92.3 & \textit{93.4} & \textbf{96.4} \\
  \textbf{IncRes-v2$_{ens}$} & 84.8 & 83.4 & 90.1 & 75.5 & 88.6 & \textit{91.1} & \textbf{94.6} \\
  
\bottomrule
\end{tabular}
}%
\end{table*}

Furthermore, \bestcombo{} and \bestcomboadv{} achieved the highest
performance also when transferring attacks to
adversarially trained \dnn{}s (\tabref{tab:normal}; 85.3\% and
91.8\% avg.\ transferability for \bestcombo{} and
\bestcomboadv{}, respectively, compared to $\le$82.5\%
avg.\ transferability for the baselines). Transferring \AE{}s crafted
using an ensemble of models increased transferability further
(\tabref{tab:ens2adv}; respectively, 93.4\% and 95.7\%
avg.\ transferability for \bestcombo{} and \bestcomboadv{}). Per a
paired t-test, the differences between 
\bestcombo{} and \bestcomboadv{} compared to the baselines over all
pairs of surrogates and targets considered were statistically
significant ($p$-value$<$0.01).

\begin{table*}[tp!]
\centering
\caption{\label{tab:defs}
  Transferability (\%) from an ensemble of normally trained
  surrogates (Inc-v4, Res-50, Res-101 and Res-152) to
  defended \imagenet{} models.
}
\centerline{%
\begin{tabular}{l | r r r r | r  r }
  \toprule

\textbf{Defense} & \textbf{\ASDTM{}} & \textbf{\dst{}} & \textbf{\sdtv{}} & \textbf{\undp{}} & \textbf{\bestcombo{}} & \textbf{\advancedultimatecombo{}} \\ \midrule
\textbf{\bitred{}} & 88.6 & 88.2 & 94.8 & 94.9 & \textbf{96.0} & \textit{95.5} \\
\textbf{\nrp{}} & 51.0 & 54.9 & \textbf{80.0} & 27.9 & \textit{65.3} & 55.8 \\
\textbf{\randsmooth{}} & 87.3 & 84.8 & 90.6 & 85.5 & \textit{88.5} & \textbf{95.6} \\
\textbf{\arandsmooth{}} & 65.4 & 62.9  & 66.5 & 61.9 & \textit{67.0} & \textbf{71.9} \\

\bottomrule
\end{tabular}
} %

\end{table*}

\begin{table*}[t!]

\caption{\label{tab:cifarTRS}Transferability (\%) on \cifarten{}
  from a normally trained VGG surrogate to an ensemble of Res
  \dnn{}s trained via \TRS{}.
}
 \centerline{
\begin{tabular}{l|r r r r | r r r }
  \toprule
     \textbf{Epsilon} &\textbf{\ASDTM{}} & \textbf{\SDTM{}} & \textbf{\sdtv{}} & \textbf{\undp{}} & \textbf{\gsdt{}} & \textbf{\bestcombo{}}& \textbf{\advancedultimatecombo{}} \\ \midrule
   0.02 & 23.7 &  26.6 & 22.1&24.9 & {28.2} & \textit{32.7} & \textbf{46.5} \\
   0.04 & 45.8  & 46.7 & 41.9 &52.6 &{52.8} & \textit{59.4} & \textbf{80.0} \\
\bottomrule
\end{tabular}
}

\end{table*}

\parheading{Additional Defenses}
Besides adversarially trained models, we evaluated
transferability against five standard %
defenses. Two defenses, bit reduction (\bitred{})~\cite{xu2018feature}
and neural representation purification
(\nrp{})~\cite{naseer2020self}, seek to sanitize
adversarial perturbations. Two others, randomized smoothing
(\randsmooth{})~\cite{cohen2019certified} and randomized smoothing with
adversarial training (\arandsmooth{})~\cite{salman2019provably} offer
provable robustness guarantees. Last, \TRS{} leverages an ensemble of
smooth \dnn{}s trained to have misaligned gradients, to defend
attacks~\cite{yang2021trs}. We evaluated all defenses but \TRS{} on
\imagenet{}. We used the defenses with default parameters, 
and transferred \AE{}s
crafted against an ensemble of normally trained models. The results
are shown in \tabref{tab:defs}. Except for \nrp{}, where \sdtv{}
attained the highest performance, %
\bestcomboadv{} outperformed the baselines by large
margins. %
Following~\cite{yang2021trs}, we tested
\TRS{} on \cifarten{} with adversarial perturbation norms
$\epsilon\in\{0.02, 0.04\}$.
\ultimatecombo{}'s variants did best against this
defense as well (\tabref{tab:cifarTRS}).

\begin{table*}[t!]

\caption{\label{tab:googleAPI}Transferability (\%) on \imagenet{}
  from an esemble of normally trained surrogates (Inc-v4, Res-50,
  Res-101 and Res-152) to Google Cloud Vision.}
\centerline{ %
\begin{tabular}{r r r r | r r}
  \toprule
     \textbf{\ASDTM{}} & \textbf{\SDTM{}} & \textbf{\sdtv{}} & \textbf{\undp{}}& \textbf{\bestcombo{}}&\textbf{\bestcomboadv{}} \\ \midrule
   76.2 & 73.4 &  72.6 & \textit{81.3} & {76.5} & \textbf{82.1} \\
\bottomrule
\end{tabular}
} %
\end{table*}

\parheading{Commercial System}
Last, we  tested attacks against Google Cloud Vision to
simulate an even more realistic setting where target models are not
publicly available and have been trained on a distinct training set
compared to the surrogate. To estimate transferability, we
classified the benign \imagenet{} images via Google's API and
calculated the ratio of \AE{}s that were classified differently than
their benign counterparts. \tabref{tab:googleAPI} presents the
result. %
\bestcomboadv{} outperformed \emph{all} baselines, with 0.8\%--9.5\% higher
transferability rates.

\subsection{Parallel vs.\ Serial Composition}
\label{sec:res:ParVsSer}
\label{sec:ParVsSer}

To corroborate that parallel composition is useful for transferability, we compared it to 
the previously established serial composition of the augmentation methods. In this 
experiment, we tested the augmentations included in \ultimatecombo{} and measured
transferability on \imagenet{}, with Inc-v3 as surrogate and all other models as targets. 
The results showed that serial composition led to markedly lower transferability than 
parallel composition---26.1\% vs. 88.9\% average transferability rates.
A potential explanation is that serial composition deteriorates image quality, 
significantly dropping the surrogate's benign accuracy (e.g., 21.0\% vs. 82.4\% benign 
accuracy on Inc-v3 on augmented images when using serial and parallel composition, 
respectively), such that the adversarial directions produced do not generalize to the 
target models.%

\subsection{Attack Run-Times}
\label{sec:time}

\mifgsm{}'s time complexity is predominated by the gradient
computation steps. Accordingly, the attacks' run times are directly
affected by the number of samples 
the augmentation methods create (i.e., samples emitted by $D(\cdot)$ in
\algoref{alg:dataaug}): The more samples emitted by the augmentation
method, the more time would be spent computing
gradients for updating the \AE{}s in each iteration,
thus increasing \AE{}s' generation time. The empirical measurements
corroborate this intuition (\tabref{tab:timing}). 
\dstonly{} augments \mifgsm{} with the least samples (except for
\undp{} that runs for 100 iterations) and runs for 10 iterations, 
leading to the fastest attack (\dst{}). \gsdt{} is the second fastest
attack, while \bestcombo{} is slower than \admixdt{} but substantially
faster ($\times$2.44) than \sdtv{}. \bestcomboadv{} augments attacks
with the 
largest number of samples, thus increasing run time the
most. %
We note that no particular effort was
invested to make \gsdt{}, \bestcombo{}, and \bestcomboadv{} more
time-efficient (e.g., stacking augmented 
samples for parallel computation, similarly to \admixdt{}). Moreover,
since transferability-based attacks generate \AE{}s offline, and only
once per surrogate model,
attack run-time is a marginal consideration for selecting an
attack compared to transferability. %
Lastly, for improved time efficiency, an adversary with strict time
requirements may consider using a subset of samples augmented by
\bestcomboadv{}'s. We found that when sampling 30 of the 165 augmented
images (\subsetattack{30} in \tabref{tab:timing}), we obtain an attack
$\times{}7.27$ faster than \bestcomboadv{} with roughly the
same transferability. By contrast, increasing the number of samples in
baseline attacks only slows them down with barely no increase
in transferability. %

\begin{table*}[t!]
\setlength\tabcolsep{3pt}
\small
\centering
\caption{\label{tab:timing}The number of samples augmented and the
  avg.\ time of crafting an \AE{} %
  per attack. Times were measured on an Nvidia A5000 GPU,
  on \imagenet{}, when attacking an
  Inc-v3, and averaged for 1K samples. For best
  transferability rates, \undp{} %
  was run
  for 100 iterations, while other attacks were run for 10.
}
\centerline{ %
\begin{tabular}{l | r r r r | r r r r}
  \toprule
  \textbf{Attack} & \textbf{\ASDTM{}} & \textbf{\SDTM{}} &
  \textbf{\sdtv{}} & \textbf{\undp{}} & \textbf{\gsdt} &
  \textbf{\bestcombo} & \textbf{\subsetattack{30}} & \textbf{\advancedultimatecombo{}} \\ \midrule
  \textbf{Augmented samples} & 15 & 5 & 105 & 1 & 10 & 30 & 30 & 165 \\
  \textbf{Time (s)} & 1.68 & \textbf{0.72} & 11.29 & 1.65 & \textit{1.10} &
  3.62 & 5.54 & 40.27 \\

  \bottomrule
\end{tabular}} %

\end{table*}

\subsection{Supporting Theory}
\label{sec:res:theorysupport}

\begin{figure}[t!]
    \centering
    \includegraphics[width=0.5\textwidth]{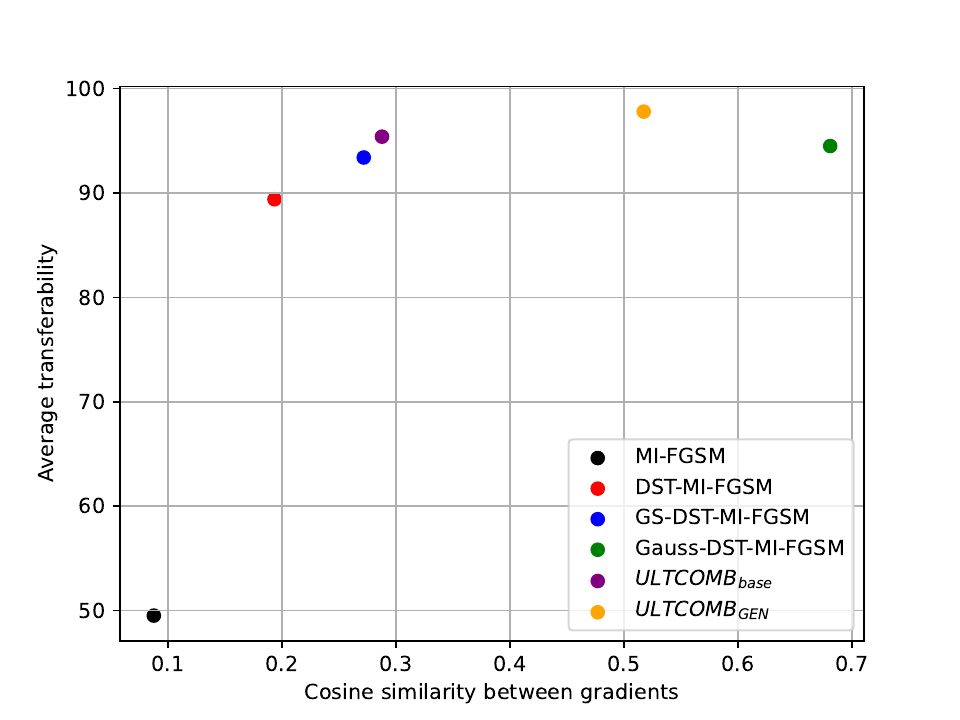}
    \caption{\label{fig:cosinesimilarity}The relationship between
      the similarity between consecutive gradients computed in attacks
      and transferability rates. Results obtained using Inc-v3 as a
      surrogate and the remaining \imagenet{} models as targets.
      Gauss-\dst{} refers attacks composing Gaussian noise and
      \dst{} and uses the  same number of augmented images as \bestcomboadv{}. Notice how \bestcombo{} and \bestcomboadv{} lead to
      higher similarity between gradients than most attacks, and how
      the transferability rates tend to increase as
      the gradient similarity increases.
    }
\end{figure}

Our theoretical results show that augmentations often smoothen the
gradients computed by the surrogate, thus decreasing the effect of
the surrogate's peculiarities on \AE{}s. In turn, this is likely to
lead to improve the generalization of \AE{}s from the surrogate to
target models. Our empirical findings support this
theory. Particularly, using the Inc-v3 models as a surrogate, we
tested the (cosine) similarity between gradients throughout attack
iterations when integrating different augmentations into
attacks. %
Additionally, we measured the avg.\ transferability of attacks to the
remaining \imagenet{} models to assess the relationship between
smoothness and transferability. We expected that Gaussian-noise
augmentations as well as compositions of multiple augmentations (as in
\bestcombo{} and \bestcomboadv{}) would lead to smoother gradients,
with higher cosine similarities across attack iterations. Moreover, we
anticipated that transferability would be higher when the gradients
are smoother. As shown in \figref{fig:cosinesimilarity}, both of these
expectations held---the cosine similarities between consecutive
gradients were higher for \bestcombo{} and \bestcomboadv{} than for
all other attacks except for when Gaussian noise is integrated, and the
transferability rates tended to be higher as the cosine similarities
increased.
We note that the method integrating Gaussian noise attains the highest
cosine similarities with slightly lower transferability than
\bestcombo{}. %
Accordingly, in addition to smoothness, there may be other factors
that can influence transferability. We leave the identification of
these factors for future work.

\section{Conclusion}

Leveraging a newly proposed means to integrate %
data augmentations
into attacks (i.e., parallel composition) and systematically studying
a broad range of augmentation methods, our work uncovered a  mostly monotonic
relationship between augmentations and transferability. Our approach
helped us identify compositions
(\bestcombo{} and \bestcomboadv{}) that outperform prior techniques
integrated into attacks; these should be
considered as a standard baseline in follow-up work on transferability. 
Our work also puts forward empirically supported theoretical
explanations for why augmentations help boost transferability.

\begin{acks}
This work has been supported in part
by grant No.\ 2023641 from the United States-Israel Binational Science
Foundation (BSF);
by a grant from the Blavatnik Interdisciplinary Cyber Research Center
(ICRC);
by Intel\textregistered{} via a Rising Star Faculty Award;
by the Israel Science Foundation (grant number 1807/23);
by a gift from KDDI Research;
by Len Blavatnik and the Blavatnik Family foundation;
by a Maof prize for outstanding young scientists;
by the Ministry of Innovation, Science \& Technology, Israel (grant
number 0603870071);
by NVIDIA via a hardware grant;
and by a grant from the Tel Aviv University Center for AI and Data
Science (TAD).
\end{acks}

\bibliographystyle{ACM-Reference-Format}
\bibliography{main.bbl}

\appendix
\section{Other Augmentations Studied}
\label{app:otheraug}

Besides the 12
augmentations %
in \secref{sec:dataaug:new}, we considered the following 34
augmentations.

\textbf{\textit{Color-space}}:
\textit{(1)} \textit{K-means color quantization};
\textit{(2)} \textit{Jpeg compression};
\textit{(3)} \textit{Voronoi};
\textit{(4)} \textit{Random invert};
\textit{(5)} \textit{Random posterize};
\textit{(6)} \textit{Random solarize};
\textit{(7)} \textit{Random equalize};
\textit{(8)} \textit{Superpixels};
\textit{(9)} \textit{Guassian noise};
\textit{(10)} \textit{Emboss};
\textit{(11)} \textit{fancy PCA (\fancyPCA{})}.

\textbf{\textit{Random Deletion}}:
\textit{(1)} \textit{Dropout};
\textit{(2)} \textit{Random crop}.

\textbf{\textit{Kernel Filters}}:
\textit{(1)} \textit{Gaussian blur};
\textit{(2)} \textit{Average blur};
\textit{(3)} \textit{Median blur};
\textit{(4)} \textit{Bilateral blur};
\textit{(5)} \textit{Motion blur};
\textit{(6)} \textit{MeanShift blur};
\textit{(7)} \textit{Edge detect};
\textit{(8)} \textit{Canny};
\textit{(9)} \textit{Average pooling};
\textit{(10)} \textit{Max pooling};
\textit{(11)} \textit{Min Poling};
\textit{(12)} \textit{Median pooling}.

\textbf{\textit{Style Transfer}}:
\textit{(1)} \textit{Clouds};
\textit{(2)} \textit{Fog};
\textit{(3)} \textit{Frost};
\textit{(4)} \textit{Snow};
\textit{(5)} \textit{Rain}.

\textbf{\textit{Spatial}}:
\textit{(1)} \textit{Random perspective};
\textit{(2} \textit{Elastic transform};
\textit{(3)} \textit{Random vertical flip};
\textit{(4)} \textit{Patch shuffle}.

\begin{table*}[t!]
\caption{\label{tab:app:imageneteps}Transferability rates (\%) on \imagenet{},
  from a Inc-v3 surrogate to other normally trained models, with
  perturbation norms $\epsilon\in\{\frac{8}{255}, \frac{24}{255}\}$
  other than the default $\epsilon{}=\frac{16}{255}$.} 
\centerline{
\resizebox{\textwidth}{!}{
\begin{tabular}{l | l | r r r r r r r r r r r r | r}
\toprule

  \textbf{$\epsilon$} & \textbf{Attack}& \textbf{Inc-v3}& \textbf{Inc-v4}&\textbf{Res-152}&\textbf{IncRes-v2}&\textbf{ Res-50}&\textbf{ Res-101}&\textbf{ \vit{}}&\textbf{ \mnas{}} & \textbf{Inc-v3$_{adv}$}& \textbf{Inc-v3$_{ens3}$}&\textbf{Inc-v3$_{ens4}$}&\textbf{ IncRes-v2$_{ens}$}&\textbf{Avg.}\\ \midrule
  \multirow{5}{*}{8/255}&\ASDTM{}&98.7&75.0&60.6&69.3&68.3&62.7&33.5&67.8&57.6&52.5&51.8&36.1&57.7\\
     &\SDTM{}&\textbf{99.7}&77.1&62.9&70.9&69.8&64.8&34.1&71.1&61.4&55.0&54.7&35.9&59.8\\
     &\sdtv{}&99.5&80.0&68.1&75.1&74.8&68.9&39.7&74.5&68.6&64.4&64.5&47.7&66.0\\
    &\undp{}&99.4&\textbf{92.2}&\textbf{82.4}&\textbf{90.6}&\textbf{85.6}&\textbf{85.1}&32.3&\textit{86.6}&48.7&43.9&40.8&20.3&64.4\\

& \gsdt{}&{99.5}&{81.3}&{71.1}&{77.6}&{77.3}&{72.0}&37.0&75.4&69.2&63.9&63.9&45.2&66.7\\

 & \bestcombo{}&\textbf{99.7}&{86.0}&{75.0}&{81.8}&{81.2}&{76.3}&{40.8}&{82.5}&{72.7}&{67.2}&{64.2}&{46.2}&\textit{70.4}\\
  & \advancedultimatecombo{}&\textbf{99.7}&\textit{89.8}&\textit{80.7}&\textit{87.1}&\textit{84.3}&\textit{81.5}&\textbf{52.6}&\textbf{87.8}&\textbf{80.7}&\textbf{77.3}&\textbf{75.4}&\textbf{58.4}&\textbf{77.8}\\
  \midrule
   \multirow{5}{*}{24/255}&\ASDTM{}&99.9&97.1&93.7&95.9&94.8&94.4&73.4&94.3&87.1&88.6&88.2&79.9&89.8\\
     &\SDTM{}&\textbf{100.0}&97.8&93.6&96.8&95.1&93.9&74.9&95.9&84.6&90.5&89.3&77.2&90.0\\
     &\sdtv{}&\textbf{100.0}&97.9&95.0&97.1&96.2&95.6&78.4&97.1&89.3&93.2&91.9&84.0&92.3\\

     &\undp{}&99.8&98.9&{97.3}&{98.9}&{97.9}&98.0{}&78.8&\textit{97.6}&84.0&82.9&78.4&59.8&88.4
\\

& \gsdt{}&\textbf{100.0}&98.6&95.7&98.3&96.7&97.0&86.7&97.5&94.5&95.9&94.9&90.4&95.1\\

 & \bestcombo{}&\textbf{100.0}&\textit{99.3}&\textit{97.4}&\textit{99.2}&\textit{97.5}&\textit{98.4}&\textit{83.8}&\textit{99.0}&\textit{92.5}&\textit{95.1}&\textit{95.0}&\textit{86.9}&\textit{94.9}\\
   & \advancedultimatecombo{}&\textbf{100.0}&\textbf{100.0}&\textbf{99.8}&\textbf{100.0}&\textbf{99.6}&\textbf{99.9}&\textbf{94.1}&\textbf{99.5}&\textbf{97.9}&\textbf{99.3}&\textbf{98.4}&\textbf{95.2}&\textbf{98.5}\\

\bottomrule
\end{tabular}
}
}
\end{table*}
\section{Attack and Augmentation Method Parameters}
\label{app:augparams}

Similarly to standard practice~\cite{wang2021boosting}, we set
the \mifgsm{}  decay factor $\mu$=1.0 , and attacks' number of
iterations $T$=10, and their step-size $\alpha=\frac{\epsilon}{T}$. 
The only exception is \undp{} found to attain low transferability rates 
with these parameters. Therefore, for \undp{}, we use default parameters
found to achieve best performance by Li et al.~\cite{li2023towards}: 
$T$=100 and $\alpha=\frac{1}{255}$.

We mostly used default or commonly used parameters of augmentation
methods. For \colorjitter{}, we performed random adjustments of image
hue $\in[-0.5,0.5]$, contrast $\in[0.5,1.5]$, saturation
$\in[0.5,1.5]$, and brightness$\in[0.5,1.5]$. For \cutout{}, we
replaced values in selected regions with zeros, and the portion of
masked areas compared to image dimensions lied in $[0.02, 0.4]$, with
aspect ratios $\in[0.4, 2.5]$. In comparison, for \randerase{}, the
dimension of masked areas relatively to the image dimensions lied in
$[0.02, 0.2]$, with aspect ratios $\in[0.3, 3.3]$. For \sharpen{},
we used the following edge-enhancement mask:
\[
  \left[ {\begin{array}{ccc}
    -0.5 & -0.5&-0.5 \\
    -0.5 & 5.0 &-0.5 \\
    -0.5 & -0.5&-0.5 \\
  \end{array} } \right].
\]
For diverse inputs, images were transformed with probability 0.5.
For the \admix{} operation, consistently with \cite{Wang21Admix}, we
randomly sampled three images from other categories for mixing as part
of the \admixdt{} attack. However, for the interest of computational
efficiency, we use only one image for mixing when composing \admix{}
with other augmentation methods. We did not find that mixing with
fewer images harmed performance. In fact, it even improved transferability
in some cases. Finally, in \cutmix{}, we picked the top left
coordinate $(r_x, r_y)$, the width, $r_w$, and height, $r_h$, of the
region to be replaced, using the formulas:
$$
r_x \sim \mathcal{U}(0, W), \quad r_w=W \sqrt{1-\lambda},\\
$$
$$
r_y \sim \mathcal{U}(0, H), \quad r_h=H \sqrt{1-\lambda},
$$
where $\mathcal{U}$ is the uniform distribution, $W$ is the image
width, $H$ is the image height, and $\lambda$ is a parameter set to 0.5.

In an attempt to enhance transferability further, we optimized the
parameters of a few augmentation methods we considered via grid
search. Except for the Gaussian kernel's size used in
translation-invariant attacks~\cite{dong2019evading}, we found that
the selected parameters had little impact on
transferability. Specifically, for translations, 
after considering Gaussian kernels of sizes $\in\{5\times{}5,
7\times{}7, 9\times{}9\}$, we set the default to $7\times{}7$, except
for \admixdt{}, for which the $9\times{}9$ kernel performed best. 
The results show that our choice of \admix{} parameters ($m$=1
and Gaussian kernel size of $9\times{}9$) improves its performance.
For \greyscale{}, we found $\omega_R$, $\omega_G$,
and $\omega_B$ had little impact on transferability, as long as the
weight assigned to each channel was $>$0.1. Accordingly, we set
$\omega_R$, $\omega_G$, and $\omega_B$ to 0.299, 0.587, and 0.114,
respectively, per commonly used values (e.g., in the Python
\texttt{PyTorch} package\footnote{\url{https://bit.ly/3ynCyUD}}). Finally,
for \chanshuffle{}, we only swapped the blue and green channels, as
this led to a minor improvement compared to swapping all three channels.

Finally, we clarify that each of our attack combinations emits the original
image once, alongside the transformed images. Moreover, when
aggregating the gradients, the gradients of the original and
transformed images are assigned equal weights. We tested whether
weighting the gradients differently (e.g., assigning higher or lower
weight to the original sample) can help improve transferability using
the \greyscale{} method. However, we found that equal weights attained
the best results.

\section{Augmentations Included in \bestcomboadv{}}
\label{app:bestcomboadv-augs}

The following augmentations are composed together in \bestcomboadv{}:
\dst{}; \cutmix{}; random crop; random rotation; \colorjitter{};
Gaussian blur; random affine; random perspective; \fancyPCA{}; elastic
transform; horizontal flip; \admix{}; random invert; random
solarize; \autoaug{}; \neuraltrans{};  \cutout{}; JPEG compression;
K-means color quantization; superpixels; average blur; median blur;
motion blur; MeanShift blur; edge detect; canny; average pooling; min
pooling; patch shuffle; fog; frost; rain; Guassian noise. 

\section{Transferability Rates on \imagenet{}}
\label{app:fullresults}
\label{app:imagenet}

\tabref{tab:app:imageneteps} shows the transferability rates on
\imagenet{} from Inc-v3 to other normally traiend models with varied
perturbation norms (i.e., values of $\epsilon$).

\end{document}